\documentclass[hyperref]{article}

\usepackage{microtype}
\usepackage{graphicx}
\usepackage{subfigure}
\usepackage{booktabs} 

\usepackage{hyperref}

\usepackage[accepted]{icml2022}

\usepackage{amsmath}
\usepackage{amssymb}
\usepackage{mathtools}
\usepackage{amsthm}

\usepackage[capitalize,noabbrev]{cleveref}

\theoremstyle{plain}
\newtheorem{theorem}{Theorem}
\newtheorem{proposition}{Proposition}

\theoremstyle{remark}

\renewcommand{\algorithmicrequire}{ }

\usepackage[textsize=tiny]{todonotes}

\usepackage{enumitem}

\usepackage{multicol}
\usepackage{multirow} 
\usepackage{xspace}
\newcommand{\mname}{\texttt{CSNN}\xspace}
\newcommand{\mnameone}{\texttt{A2D}\xspace}
\newcommand{\mnametwo}{\texttt{RE}\xspace}

\icmltitlerunning{Sun et al. CSNN}

\begin{document}

\twocolumn[
\icmltitle{Curriculum Design Helps Spiking Neural Networks to Classify Time Series}




\begin{icmlauthorlist}
\icmlauthor{Chenxi Sun}{1,2,3}
\icmlauthor{Hongyan Li}{1,2,3}
\icmlauthor{Moxian Song}{1,2,3}
\icmlauthor{Derun Cai}{1,2,3}
\icmlauthor{Shenda Hong}{4,5}

\end{icmlauthorlist}

\icmlaffiliation{1}{Key Laboratory of Machine Perception (Ministry of Education), Peking University, Beijing, China.}
\icmlaffiliation{2}{National Key Laboratory of General Artificial Intelligence, Beijing, China.}
\icmlaffiliation{3}{School of Intelligence Science and Technology, Peking University, Beijing, China.}
\icmlaffiliation{4}{National Institute of Health Data Science, Peking University, Beijing, China.}
\icmlaffiliation{5}{Institute of Medical Technology, Health Science Center of Peking University, Beijing, China.}

\icmlcorrespondingauthor{Chenxi Sun}{sun\_chenxi@pku.edu.cn}

\icmlkeywords{Machine Learning, ICML}

\vskip 0.3in
]



\printAffiliationsAndNotice{}  

\begin{abstract}

Spiking Neural Networks (SNNs) have a greater potential for modeling time series data than Artificial Neural Networks (ANNs), due to their inherent neuron dynamics and low energy consumption. However, it is difficult to demonstrate their superiority in classification accuracy, because current efforts mainly focus on designing better network structures. In this work, enlighten by brain-inspired science, we find that, not only the structure but also the learning process should be human-like. To achieve this, we investigate the power of Curriculum Learning (CL) on SNNs by designing a novel method named \mname with two theoretically guaranteed mechanisms: The active-to-dormant training order makes the curriculum similar to that of human learning and suitable for spiking neurons; The value-based regional encoding makes the neuron activity to mimic the brain memory when learning sequential data. Experiments on multiple time series sources including simulated, sensor, motion, and healthcare demonstrate that CL has a more positive effect on SNNs than ANNs with about twice the accuracy change, and \mname can increase about 3\% SNNs' accuracy by improving network sparsity, neuron firing status, anti-noise ability, and convergence speed.

\end{abstract}

\section{Introduction} \label{sec:introduction}

Despite producing positive results in the classification of time series, Artificial Neural Networks (ANNs) appear to reach the bottleneck stage: Time series exhibits dynamic evolution, whereas most ANNs are static. Although the purpose is to create brain-like dynamics, ANNs only adhere to the connection mode and do not penetrate deeply into the neurons. To model more data dynamics, even Recurrent Neural Networks (RNNs), which are inherently dynamic, need additional mechanisms \cite{ijcai2021-414}. Meanwhile, accuracy and lightness are frequently incompatible with ANNs, while ambient-assisted living represents the general trend and favors high-efficient and low-consumed models.

A growing efficient paradigm is Spiking Neural Networks (SNNs). They mimic the brain capacity, specifically the intricate dynamics of spiking neurons and the plastic synapses bridging them, and are scientifically plausible. Since each neuron is a dynamic system that can learn both the time and the order relation, it is more suited to depict evolution in real-world time series. Meanwhile, by processing the data as sparse spike events, SNNs offer low power usage.

However, SNNs have struggled to demonstrate a clear advantage over ANNs in accuracy due to the information loss during spike coding, the incomplete modeling of neurotransmitter transmission, and the inability of gradient backpropagation \cite{SRNN}. Consequently, the majority of the effort to increase the accuracy begins with these model structure-related factors. But back to the original intention, SNNs are more brain-like, not only the model structure but also the learning process should resemble that of humans. People usually learn easier knowledge before more complicated ones. Curriculum Learning (CL) has shown that an easy-to-hard training order can enhance the model performance and generalization power for ANNs \cite{DBLP:conf/icml/HacohenW19}. For SNNs, however, this study is lacking.

The capacity of CL to SNN in classifying time series is still unknown to us: ANNs' curriculum may not adapt to SNNs due to their structural differences; Theoretical principles are required to lead SNNs' curriculum; The designed curriculum should be adjusted to account for spiking features; The classification will be impacted by the way to acquire each time series in addition to the learning order among them \cite{DBLP:conf/cikm/SunSCZH022}. Different areas of the brain assist the memory for different pieces in a sequence \cite{brain}. However, by default, all of the neurons in an SNN that process various subparts of a time series are the same.

This work studies the power of CL in training SNNs for the first time. We recognize the temporal dynamics of Recurrent SNN (RSNN) are advantageous for modeling and classifying time series. Based on experimental observations about how training orders and memory modes affect SNNs via spiking trains and neuron firing, and through theoretically guaranteed insights, objectives, and procedures, we propose a Curriculum for SNN (\mname) with two mechanisms:

\vspace{-0.4cm}
\begin{itemize}[leftmargin=10 pt]\setlength{\itemsep}{3pt}
\setlength{\parsep}{0pt}
\setlength{\parskip}{0pt}
\setlength{\topsep}{0pt}
\setlength{\partopsep}{0pt}
    \item Active-to-dormant training order mechanism (\mnameone) uses the sample's activity to customize the training order. The output neuron firing frequency that corresponds to the sample's class label serves as a gauge of its activity.
    \item Value-based regional encoding mechanism (\mnametwo) uses various spiking neuron clusters to encode input spikes that are differentiated by observed values, imitating the regional pattern in the brain when memorizing sequential data.
\end{itemize}
\vspace{-0.3cm}

We show their rationality and effectiveness through theoretical analysis and experiments on three real-world datasets.

\section{Related Work} \label{sec:related work}

\subsection{Classification of Time Series (CTS)}\label{sec:CTS}

Time Series (TS) data is present in practically every task that calls for some kind of human cognitive process due to their inherent temporal ordering \cite{DBLP:journals/datamine/FawazFWIM19}. A TS dataset $\mathcal{D}=\{(X_{i},C_{i})\}_{i=1}^{N}$ has $N$ samples. A sample $X_{i}=\{x_{t}\}_{t=1}^{T}$ has $T$ observed values and time and is labeled with a class $C_{i}\in \mathcal{C}$. The CTS task is $f:X\to C$ with model $f$. Many ANNs have achieved SOTA for CTS tasks, which can be summarized into three categories: recurrent networks like LSTM \cite{pmlr-v119-zhao20c}, convolutional networks like 1D-CNN \cite{pmlr-v139-jiang21d}, and attention networks like Transformer \cite{10.1145/3447548.3467401}. But the dynamics of real-world TS actually vary, such as multiple scales, uneven time intervals, and irregular sampling \cite{DBLP:journals/corr/abs-2010-12493}. Current solutions often design additional mechanisms for ANNs artificiality rather than revise them internally.

\subsection{Spiking Neural Networks (SNNs)}\label{sec:SNN}

An SNN consists of neurons that propagate information from a pressynapse to a postsynapse. At the time $t$, as the synaptic current reaches the neuron it will alter its membrane potential $v(t)$ by a certain amount. If $v(t)$ reaches a threshold $V_{th}$, the pressynapse will emit a spike and reset $v(t)$ to $V_{0}$. In an SNN, each neuron has its own dynamics over time. This expands the modeling options for TS. Most current work for this biologically plausible model examines the network structure, such as the Leaky Integrate and Fire (LIF) model \cite{Ror2006The} and the gradient surrogate updating strategy \cite{DBLP:journals/ijns/KheradpishehM20}. For modeling TS, SNN \cite{DBLP:conf/ijcnn/FangSQ20}, RSNN \cite{SRNN}, and LSNN \cite{DBLP:conf/nips/BellecSSL018} are proposed, but their accuracy is still behind that of non-spiking ANNs of approximate architecture \cite{sciadv}.

\subsection{Curriculum Learning (CL)}\label{sec:CL}

CL is motivated by the curriculum in human learning, attempts at imposing some structure on the training set. It gives a sequence of input mini-batches $\mathcal{D}\rightarrow \mathbb{B}=[\mathcal{B}_{b}]_{b=1}^{B}$. Two subtasks are scoring $f_{s}$ and pacing $f_{p}$: $f_{s}$ ranks samples, $f_{s}:\mathcal{D}\rightarrow \mathbb{R}=[(X_{i},C_{i})]_{i=1}^{N}$, if $f_{s}(i)<f_{s}(j), (X_{i},C_{i})\succ(X_{j},C_{j})$, such as knowledge transfer and self-taught strategies \cite{DBLP:conf/nips/CastellsWR20}; $f_{p}$ determines which sample is presented to the network by giving a sequence of subsets $\mathbb{B}$ of size $f_{p}(b)=|\mathcal{B}_{b}|$, such as single-step and exponential functions \cite{DBLP:conf/ijcai/LinC0J022}. Most work shows that the easy-to-hard training order outperforms the random shuffling \cite{DBLP:conf/icml/HacohenW19}. For CTS tasks, a few methods such as confidence-based CL are proposed \cite{DBLP:conf/cikm/SunSCZH022}. But the examination of curriculum tailored to TS' characteristics is not extensive. Most importantly, none of them can guarantee that they would have a positive effect on SNNs because they are all ANN-specific. Research on CL for SNNs is anticipated.

\section{Methods}\label{sec:method}

As shown in Figure \ref{fig:method}, the proposed curriculum consists of two mechanisms\footnote{All theorems, propositions and proofs involved in the proposed method will be thoroughly discussed in Appendix \ref{sec:appendix}.}: The active-to-dormant training order mechanism (\mnameone) is in Section \ref{sec:method1}; The value-based regional encoding mechanism (\mnametwo) is in Section \ref{sec:method2}. Each of them is introduced through a theoretically supported process that includes \textsc{insight}, \textsc{objective}, and \textsc{procedure}.

\begin{figure*}[t]
\centering
\includegraphics[width=0.9\linewidth]{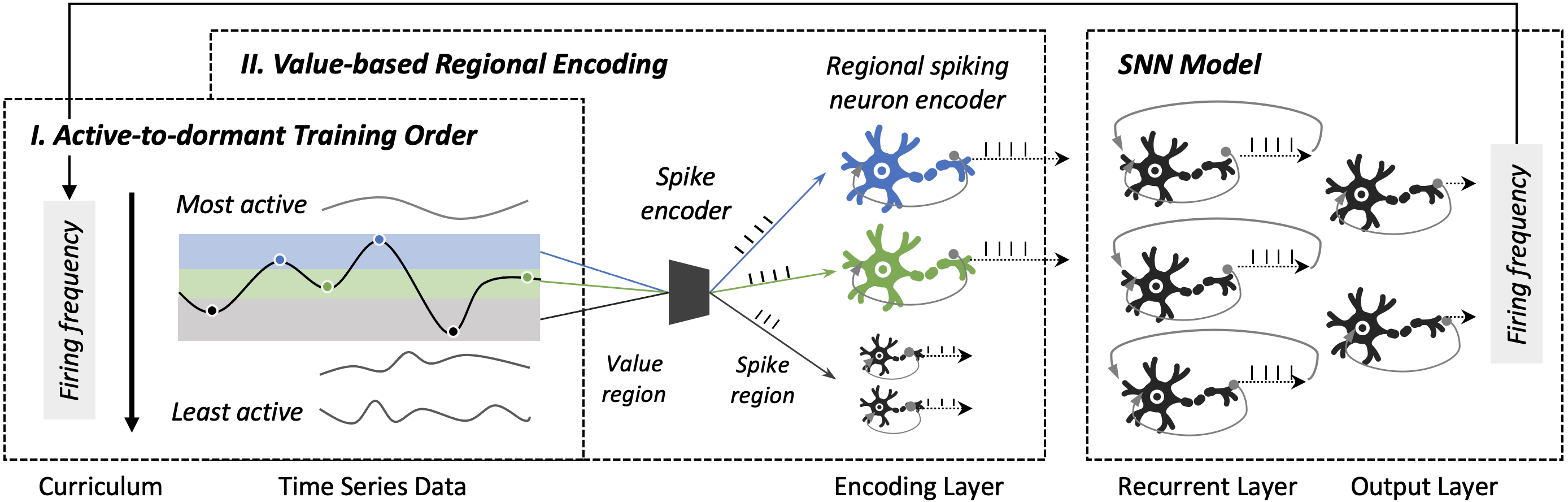}\vspace{-0.3cm}
\caption{Curriculum of Time Series Data for Recurrent Spiking Neural Network}
\label{fig:method}
\end{figure*}

\subsection{Active-to-dormant Training Order (\mnameone)}\label{sec:method1}

\subsubsection*{Insight: Different training orders make the spiking neuron output different spike trains} 

We adopt the a widely used LIF neuron. Each input spike induces a charge in the neuron’s membrane potential, called a Post Synaptic Potential (PSP). In Equation \ref{eq:PSP}, $t_{i}$ is the arrival time of i-th input spike $s_{i}$. $K(t)$ is the synapse kernel, where $\tau_{m}$ and $\tau_{s}$ are time constants. $V_{0}=\frac{\eta}{\eta-1}$ and $\eta=\frac{\tau_{m}}{\tau_{s}}$ scale the maximum value of $K(t)$ to 1.
\begin{equation}\label{eq:PSP}
\small
    \mathrm{PSP}(t)=\sum_{t_{i}}^{t_{i}<t}K(t-t_{i})s_{i}, \ K(t)=V_{0}(e^{-\frac{t}{\tau_{m}}}-e^{-\frac{t}{\tau_{s}}})
\end{equation}

The neuron accumulates all input PSPs and then forms the membrane potential $v(t)$. In Equation \ref{eq:vt}, $N_{I}$ is the number of input synapse. $w_{i}$ is the weight associated with each input synapse. $t_{s}<t$ is the time when the neuron generates an output spike. And the membrane potential is decreased by a factor of the threshold voltage $V_{th}$. This serves as the reset mechanism at the time of spike. 
\begin{equation} \label{eq:vt}
\small
v(t)=\sum_{i}^{N_{I}}w_{i}\mathrm{PSP}_{i}(t)-V_{th}\sum_{t_{s}}^{t_{s}<t}e^{-\frac{t-t_{s}}{\tau}}
\end{equation}

We define the input spike train $I[t]$ and the output spike train $O[t]$ as sequences of time shifted Dirac delta function, where $s[t]=1$ denotes a spike received at time $t$, otherwise $s_{i}=0$. $y[t]>0$ satisfies $v(t)>V_{th}$, otherwise $y[t]=0$ 
\begin{equation} \label{eq:spike train}
\small
I[t]=\sum_{i}^{t}s[n]\delta(t-i),\quad O[t]=\sum_{i}^{t}y[n]\delta(t-i)
\end{equation}

\begin{theorem} \label{theo:snn1}
In SNN, different input orders of spike trains make the neuron output different spike trains.
\end{theorem}

\begin{proof}
Most simply, assuming an SNN with parameters $w$ initialized randomly under a Gaussian distribution $\mathcal{N}(\mu,\sigma^2), \frac{\mu}{2}<V_{th}<\mu$, an $l$-length input sequence with all-one spike $S_{0}=(1)^{l}$, and an $l$-length input sequence with half-one spike $S_{1}=(0)^{\frac{l}{2}}\land (1)^{\frac{l}{2}}$, when giving two training order $S_{0} \rightarrow S_{1}$ and $S_{1} \rightarrow S_{0}$ to SNN, the membrane potential $v$ of one neuron in the first layer may output two different spike trains $O_{0}=(1,1)\neq O_{1}=(0,1)$.
\end{proof}

The importance of the training order is demonstrated by Theorem \ref{theo:snn1}. And different output spike trains will result in different update times and degrees of parameters, finally forming different SNN \cite{pmlr-v162-chen22ac}. Meanwhile, through the proof, the training order also affects whether a sample participates in updating parameters of SNN. For example, in $S_{1} \rightarrow S_{0}$, $S_{1}$ does not make the neuron fire so that the SNN is not updated. Thus, the training order may also affect training stability, efficiency, and noise resistance.

\subsubsection*{Objective: CL requires consistency between objective function and sampling probability}

We define a maximum function $U_{\vartheta}(X_{i})$ as the objective of CTS task. The hyper-parameter set $\vartheta$ represents a model and its different settings will produce different model functions.
\begin{equation} \label{eq:loss_L_theta}
\small
     \mathcal{U}(\vartheta)=\hat{\mathbb{E}}[U_{\vartheta}]=\sum_{i=1}^{N}U_{\vartheta}(X_{i}), \
    \Tilde{\vartheta}=\arg\max\limits_{\vartheta}\mathcal{U}(\vartheta)
\end{equation}

CL provides a Bayesian prior for data sampling $p_{i}=p(X_{i})$. For example, a non-increasing function of the difficulty level of $X_{i}$, $p(X_{i}) = \frac{1}{N}$ for $N$ training samples whose difficulty score $<\epsilon$, and $p(X_{i}) = 0$ otherwise. The threshold $\epsilon$ is determined by the pacing function which drives a monotonic increase in the number $B$.
\begin{equation} \label{eq:loss_U_theta}
\small
     \mathcal{U}_{p}(\vartheta)=\hat{\mathbb{E}}_{p}[U_{\vartheta}]=\sum_{i=1}^{N}U_{\vartheta}(X_{i})p(X_{i})
\end{equation}

\begin{theorem} \label{theo:cl1}
The difference between objectives $\mathcal{U}_{p}$ and $\mathcal{U}$, which are computed with and without curriculum prior $p$, is the covariance between $U_{\vartheta}$ and $p$.
\begin{equation} \label{eq:U_theta}
\small
    \mathcal{U}_{p}(\vartheta)=\mathcal{U}(\vartheta)+\hat{\mathrm{Cov}}[U_{\vartheta},p]
\end{equation}
\end{theorem}

\begin{proof} Equation \ref{eq:U_theta} is obtained from Equation \ref{eq:loss_U_theta}: \small{$\mathcal{U}_{p}(\vartheta)=\sum_{i=1}^{N}U_{\vartheta}(X_{i})p(X_{i})=\sum_{i=1}^{N}U_{\vartheta}(X_{i})p(X_{i})-2N\hat{\mathbb{E}}(U_{\vartheta})\hat{\mathbb{E}}(p)+2N\hat{\mathbb{E}}(U_{\vartheta})\hat{\mathbb{E}}(p)= \sum_{i=1}^{N}(U_{\vartheta}(X_{i})-\hat{\mathbb{E}}[U_{\vartheta}])(p_{i}-\hat{\mathbb{E}}[p])+N\hat{\mathbb{E}}(U_{\vartheta})\hat{\mathbb{E}}(p)=\hat{\mathrm{Cov}}[U_{\vartheta},p]+\mathcal{U}(\vartheta)$}
\end{proof}

\begin{theorem} \label{theo:cl2}
In CL, the optimal $\Tilde{\vartheta}$ maximizes the covariance between $p$ and $U_{\vartheta}$: $\Tilde{\vartheta}=\arg\max\limits_{\vartheta}\mathcal{U}(\vartheta)=\arg\max\limits_{\vartheta}\hat{\mathrm{Cov}}[U_{\vartheta},p]$, satisfying:
\begin{equation}
\small
    \begin{aligned}
        &\Tilde{\vartheta}=\arg\max\limits_{\vartheta}\mathcal{U}(\vartheta)=\arg\max\limits_{\vartheta}\mathcal{U}_{p}(\vartheta)\\
        &\forall\vartheta \quad \mathcal{U}_{p}(\Tilde{\vartheta})-\mathcal{U}_{p}(\vartheta)\geq \mathcal{U}(\Tilde{\vartheta})-\mathcal{U}(\vartheta)
    \end{aligned} 
\end{equation}
\end{theorem}

\begin{proof} From Theorem \ref{theo:cl1}, \small{$\mathcal{U}_{p}(\Tilde{\vartheta})-\mathcal{U}_{p}(\vartheta)=\mathcal{U}_{p}(\Tilde{\vartheta})-\mathcal{U}(\vartheta)-\hat{\mathrm{Cov}}[U_{\vartheta},p]\geq \mathcal{U}_{p}(\Tilde{\vartheta})-\mathcal{U}(\vartheta)-\hat{\mathrm{Cov}}[U_{\Tilde{\vartheta}},p]=\mathcal{U}(\Tilde{\vartheta})-\mathcal{U}(\vartheta)$}
\end{proof}

It demonstrates that, more so than with any other $U_{\vartheta}(X)$, $p$ has a positive correlation with the optimal $U_{\Tilde{\vartheta}}(X)$. The gradients in the new optimization landscape may thus be generally steeper in the direction of the optimal parameter $\Tilde{\vartheta}$. The original problem's global optimum is present in the updated optimization landscape created by CL, sharing the trait of having a more apparent global maximum.

\subsubsection*{Procedure: The active-to-dormant training order for SNNs meets the goal of CL}

What sort of training order can accommodate SNN while achieving the CL's goal? We demonstrate that the proposed active-to-dormant training order satisfies Theorem \ref{theo:cl1}, \ref{theo:cl2}.

We measure the activity $pC_{i}$ of a TS sample $X_{i}$ based on the distribution of firing frequency in the output layer in Equation \ref{eq:pC}. $L$ is the number of network layers, $O_{i}^{L}[t]$ denotes the output of last layer, $N_{C}$ is the neuron number of last layer. Thus, an ideal curriculum is the prior corresponding to the optimal hypothesis $p_{i}=\frac{e^{pC(X_{i})}}{P}, P=\sum_{i=1}^{N}e^{pC(X_{i})}$.
\begin{equation} \label{eq:pC}
\small
pC(X_{i})=\frac{e^{\sum_{t}^{T}O_{i}^{L}[t]}}{\sum_{j=1}^{N_{c}}e^{\sum_{t}^{T}O_{j}^{L}[t]}}
\end{equation}

In fact, Equation \ref{eq:pC} is consistent with the objective of the CTS task. In SNN, the neuron in the output layer that fires most frequently represents the result. Thus, $pC_{i}$ and $N_{C}$ can be seen as the calculated probability of each class and the class number. The cross-entropy loss is defined as Equation \ref{eq:loss}, where $C_{i}$ is the class label. $L_{\vartheta}(X_{i})$ denotes the loss of hypothesis defined by $\vartheta$ when given an sample $X_{i}$. The empirical risk minimization compute the best hypothesis of $\Tilde{\vartheta}$ from the training data.
\begin{equation} \label{eq:loss}
\small
    \mathcal{L}=\hat{\mathbb{E}}[L],\
    L(X_{i})=-C_{i}\log(pC_{i})
\end{equation}
\begin{equation} 
\small
\begin{aligned}
    \Tilde{\vartheta}&=\arg\min\limits_{\vartheta}\mathcal{L}(\vartheta)=\arg\min\limits_{\vartheta}\sum_{i=1}^{N}L_{\vartheta}(X_{i})\\
    &=\arg\max\limits_{\vartheta}\exp{(-\sum_{i=1}^{N}L_{\vartheta}(X_{i}))}=\arg\max\limits_{\vartheta}\prod_{i=1}^{N}e^{-L_{\vartheta}(X_{i})}
\end{aligned} \nonumber
\end{equation}

According to Equation \ref{eq:loss}, we specify $U$ and $p$ based on the maximum likelihood estimation with probability for empirical risk minimization $P(\vartheta|X)\propto e^{-L_{\vartheta(X)}}$.
\begin{equation}\label{eq:U}
\small
    U_{\vartheta}(X_{i})=e^{-L_{\vartheta}(X_{i})}
\end{equation}
\begin{equation}\label{eq:p}
\small
    p_{i}=\frac{e^{-L_{\Tilde{\vartheta}}(X_{i})}}{\sum_{i}^{N}e^{-L_{\Tilde{\vartheta}}(X_{i})}}=\frac{U_{\Tilde{\vartheta}}(X_{i})}{P}
\end{equation}

\begin{proposition}\label{prop:1}
    When using the active-to-dormant training order, Theorem \ref{theo:cl2} holds if the variance of the maximum function is roughly constant in the relevant range of plausible parameter values.
\begin{equation}\small
\begin{aligned}
    \mathcal{U}_{p}(\Tilde{\vartheta})-\mathcal{U}_{p}(\vartheta)\geq \mathcal{U}(\Tilde{\vartheta})-\mathcal{U}(\vartheta)\quad \forall\vartheta: \mathrm{Cov}[U_{\vartheta},U_{\Tilde{\vartheta}}]\leq\mathrm{Var}[U_{\Tilde{\vartheta}}]
\end{aligned}\nonumber
\end{equation}
\end{proposition}

\begin{proof}
From Equation \ref{eq:p}, $\mathcal{U}_{p}(\vartheta)=\mathcal{U}(\vartheta)+\frac{1}{P}\mathrm{Cov}[U_{\vartheta},U_{\Tilde{\vartheta}}]$. Then, at the optimal point $\Tilde{\vartheta}$: $\mathcal{U}_{p}(\Tilde{\vartheta})=\mathcal{U}(\Tilde{\vartheta})+\frac{1}{P}\mathrm{Var}[U_{\Tilde{\vartheta}}]$; at any other point: $\mathcal{U}_{p}(\vartheta)\leq\mathcal{U}(\Tilde{\vartheta})+\frac{1}{P}\sqrt{\mathrm{Var}[U_{\vartheta}]\mathrm{Var}[U_{\Tilde{\vartheta}}]}$. Assuming a constant $b=\mathrm{Var}[U_{\vartheta}]$, Theorem \ref{theo:cl2} follows $\mathcal{U}_{p}(\vartheta)\leq\mathcal{U}(\Tilde{\vartheta})+\frac{b}{P}=\mathcal{U}_{p}(\Tilde{\vartheta})$, i.e. $\Tilde{\vartheta}=\arg\max\limits_{\vartheta}\mathcal{U}_{p}(\vartheta)$.
\end{proof}

According to Proposition \ref{prop:1}, the optimization landscape is altered in order to emphasize the contrast between the vector of optimal parameters and all other parameter values that are covariant with the optimal solution. And our CL strategy $\Tilde{\vartheta}$, active-to-dormant training order, can make the variance of sample activity $\mathrm{Var}[U_{\Tilde{\vartheta}}]$ greater. From the perspective of the loss function, the strategy of active-to-dormant training order for SNNs is consistent with that of easy-to-hard training order for ANNs in the classical CL.

\subsection{Value-based Regional Encoding (\mnametwo)}\label{sec:method2}

\subsubsection*{Insight: Sequential data are memorized by the human brain through distinct regions}

According to recent studies \cite{brain}, the brain memorizes sequential information in multiple regions with a geometrical structure rather than just one. This type of regional memory can boost spike firing. Inspired by this insight, we suggest a regional spiking mechanism based on TS values for SNN. In this way, on the basis of imitating the information transmission mechanism of the brain, SNN also mimics the sequence processing mechanism of the brain.

\subsubsection*{Objective: CL benefits from SNN's improved sample discrimination.}

Based on Proposition \ref{prop:1}, making $\mathrm{Var}[U_{\Tilde{\vartheta}}]-\mathrm{Cov}[U_{\vartheta},U_{\Tilde{\vartheta}}]$ larger will boost the advantages of utilizing CL for SNN on the CTS task. To achieve this goal, $\mathrm{Var}[U_{\Tilde{\vartheta}}]$ should be larger, and/or $\mathrm{Cov}[U_{\vartheta},U_{\Tilde{\vartheta}}]$ should be smaller, i.e., the variance of each sample's $U_{\Tilde{\vartheta}}(X_{i})$ becomes larger, or/and $U_{\Tilde{\vartheta}}(X)$ diverges from $U_{\vartheta}(X)$ more.

\begin{theorem}
 \label{theo:ri1}
The difference between the objective score with CL and that without CL will be greater if there is a larger score discrepancy between each sample.
\end{theorem}

\begin{proof} 
Without CL, since the training samples are fixed, $\small{\mathrm{Var}[U_{\vartheta}]}$ is constant; With CL, based on Equation \ref{eq:p}, as $\small{P \approx \mathbb{E}[e^{-L(X_{i})}]}$ is constant, $\small{\mathrm{Var}[U_{\Tilde{\vartheta}}]\propto \mathrm{Var}[p]}$. Thus, $\small{\mathrm{Var}[p]\uparrow\rightarrow \mathrm{Var}[U_{\Tilde{\vartheta}}]\uparrow\rightarrow \mathrm{Cov}[U_{\vartheta},U_{\Tilde{\vartheta}}]\downarrow}$
\end{proof}

Theorem \ref{theo:ri1} states that if the CL strategy increases the difference among the score value $p$ of each sample, $\mathrm{Var}[U_{\Tilde{\vartheta}}]-\mathrm{Cov}[U_{\vartheta},U_{\Tilde{\vartheta}}]$ will be larger, CL's benefits will be more clear.

\subsubsection*{Procedure: The value-based regional encoding boosts SNN's differential response to samples}

How to make SNN response to various TS samples differ more? We suggest the value-based regional encoding mechanism by replicating human brain processing sequence.

For a TS $X=\{x_{t}\}_{t=1}^{T}$, in the classical SNNs, the input spike train $I[t]$ in Equation \ref{eq:spike train} corresponding to each observed value $x_{t}$ will be input to all $N_{\mathrm{input}}$ neurons of the input layer, so all membrane potentials are Equation \ref{eq:vt} with $t=1,...,T$.  

Different from the input of classical SNNs, in value-based regional encoding mechanism, different neurons receive the input having different observed value: We divide the interval $[\min(x), \max(x)]$ into $M$ subintervals $\mathbb{I}=\{\mathcal{I}_{m}\}_{m=1}^{M}, M\leq N_{\mathrm{input}}$, thus every $\frac{N_{\mathrm{input}}}{M}$ neurons are responsible for a numerical interval and receive different input spike trains. The membrane potential of neurons having $\mathcal{I}_{m}$ are Equation \ref{eq:vt} with $t$ satisfing $x_{t}\in \mathcal{I}_{m}$.

\begin{proposition}\label{prop:2}
    The regional encoding contributes to SNNs' CL by increasing the variation of sample activity and assisting the model's quick response to input changes.
\end{proposition}

\begin{proof}
    At time $t$, we call spiking neurons with $x_{t}\in \mathcal{I}_{m}$ as excitatory neurons $\mathrm{S}_{E}$ and neurons with $x_{t}\notin \mathcal{I}_{m}$ as inhibitory neurons $\mathrm{S}_{I}$. For $\mathrm{S}_{I}$, the membrane potential is only discharged but not charged, i.e. Equation \ref{eq:vt} only has the rightmost item. Thus, we can regard it as an extreme case of feed-forward and feed-back inhibition. The inhibition changes piking rate \cite{2022BackEISNN}, the value of Equation \ref{eq:U} increases under the true class label so that the difference among computed values of Equation \ref{eq:loss_L_theta} of samples with different labels becomes larger. Thus, $\mathrm{Var}[U_{\Tilde{\vartheta}}]$ increases.
     
    The neuronal input is given by external and recurrent input $I_{i}[t]=I^{ext}_{i}[t]+I^{rec}_{i}[t]$. The external inputs have excitatory ($e$) $K_{E}=P_{C} \mathrm{S}_{E}$ and inhibitory ($i$) $K_{I}=P_{C} \mathrm{S}_{I}$ with a probability $P_{C}$. The recurrent inputs are Poisson spike trains $I_{i}^{rec}[t]=\mu_{i}^{rec}+\sigma_{i}^{rec}\xi_{i}^{rec}$ \cite{DBLP:conf/icc/TorabK01}. $\mu_{e/i}=\mu_{rec,e/i}+f_{e/i}\mu_{ext}, \sigma^{2}_{e/i}=\sigma^{2}_{ext}+\sigma^{2}_{rec,e/i}, \beta_{e/i}=\frac{\sigma^{2}_{e/i}}{\mu_{e/i}}$ are inputs' mean, variance, variance-to-mean radio. $\sigma^{2}_{ext}$ is usually very small and $\beta_{e/i}$ is a constant irrespective to external inputs. The neural firing rate is $r_{e/i}=\frac{\mu_{e/i}}{V_{th}V_{0}}\propto \mu_{ext}$. It linearly encodes the external input mean, ensuring the network’s response to input changes very fast.
\end{proof}

\subsection{Curriculum for SNNs (\mname)}

\textsc{Classification model.} Direct implementation of the SNN model defined by Equation \ref{eq:vt} is not practical. We use an incremental way to update the PSP as indicated in Equation \ref{eq:SNN}. It can be derived from the spike response model in discrete time domain. Each neuron has a recurrent mode. $l,i,j$ are layer, neuron, input index. $N_{l}$ denotes the number of neurons in $l$-th layer. $I[t], R[t], O[t]$ are input current, reset voltage, neuron output. $\mathrm{H}(x)$ is a Heaviside step function: $\mathrm{H}(x)=0,\text{if } x<0, \text{otherwise } 1$. 
\begin{equation}\label{eq:SNN}
     \scriptsize
    \begin{aligned}
        &V_{i}^{l}[t]=I_{i}^{l}[t]-V_{th}R_{i}^{l}[t], \ \ I_{i}^{l}[t]=V_{0}\sum_{j}^{M_{l-1}}w_{i,j}^{l}(M_{i}^{l}[t]-H_{i}^{l}[t])\\
        &M_{i}^{l}[t]=\alpha N_{i}^{l}[t-1]+O_{j}^{l-1}[t], \ \ H_{i}^{l}[t]=\beta H_{i}^{l}[t-1]+O_{j}^{l-1}[t]\\
        &R_{i}^{l}[t]=\gamma R_{i}^{l}[t]+O_{i}^{l}[t-1], \ \ O_{i}^{l}[t]=\mathrm{H}(V_{i}^{l}[t]-V_{th})
    \end{aligned}
\end{equation}

To encode TS into spike sequences, we utilize a population of current-based integrate and fire neurons as encoder and pre-train the encoder using a neural engineering framework \cite{DBLP:conf/ijcnn/FangSQ20}; To train SNN, we employ Equation \ref{eq:loss} as the loss function and update SNN's parameters by employing the back-propagation through time and the gradient surrogate method (sigmoid) during training process \cite{DBLP:journals/ijns/KheradpishehM20}.

\textsc{Curriculum design.} Algorithm \ref{alg:CSNN} displays the proposed approach $(\mname=\mnameone+\mnametwo)$. For \mnameone, the scoring function is based on Equation \ref{eq:p}, the pacing function is the exponential pacing in Equation \ref{eq:fp}; For \mnametwo, there are $N_\mathrm{input}$ spiking neuron receivers and $M$ TS value intervals, every $\frac{N_{\mathrm{input}}}{M}$ neurons belong to a cluster $M<N_\mathrm{input}$. 
\begin{equation}\label{eq:fs} \small
     f_{s}(X_{i})= p_{i}=\frac{e^{-L(X_{i})}}{\sum_{i}^{N}e^{-L(X_{i})}}
\end{equation}
\begin{equation}\label{eq:fp} \small
    f_{p}(m)=\min(\mathrm{start\_percent}\cdot (1+\lfloor\frac{m}{\mathrm{step\_length}}\rfloor),1)\cdot N
\end{equation}

\begin{algorithm}[t]\caption{\mname} \label{alg:CSNN}
\begin{algorithmic}[1] 
\REQUIRE \textsc{// \mnameone Mechanism}
\STATE Get sores $\mathcal{P} \leftarrow$ Equation \ref{eq:fs}
\STATE Get the sorted TS dataset $\mathbb{R}$ by $\mathcal{P}$
\STATE Initialize mini-batches $\mathbb{B}\leftarrow [\ ]$
\FOR {$b = 1$ to $B$}
    \STATE Get size $|\mathcal{B}_{b}| \leftarrow$ Equation \ref{eq:fp}
    \STATE Get mini-batch $\mathcal{B}_{b} \leftarrow \mathbb{R}[(X_{i},C_{i})]_{i=1}^{|\mathcal{B}_{b}|}$
    \STATE $\mathbb{B}\leftarrow \mathcal{B}_{b}$
\ENDFOR
\REQUIRE \textsc{// \mnametwo Mechanism}
\STATE Construct $f_{e}$ with $M$ encoder clusters 
\FOR {$b = 1$ to $B$}
    \STATE Get spiking trains $\mathbb{I} \leftarrow f_{e}(\mathcal{B}_{b})$
    \STATE Train RSNN with $\mathbb{I}$
\ENDFOR

\end{algorithmic}
\end{algorithm}

\begin{table*}[!ht]
\scriptsize
\centerline{
\resizebox{\textwidth}{!}{
\begin{tabular}{l|ccccc|cccc|cc}
\toprule 
&$\tau_{m}(\mu,\sigma)$ &$\tau_{s}(\mu,\sigma)$ &$\tau(\mu,\sigma)$  &$a$ &bias &$sp$ &$ss$ &$\eta$ &$\eta$ decay(type) &$M$ &\tiny{$\frac{N_{\mathbf{input}}}{M}$} \\
\midrule
UCR &(20,5) &(150,10) &(20,5) &20 &0(fixed)  &5\% &50 &1e-2 &.5per10(step) &[5,10] &16\\
SEPSIS &(20,5) &(150,50) &(20,5) &20 &0(fixed) &10\% &2,000 &1e-2 &.5per20(step) &5 &32\\
COVID-19 &(20,5) &(150,50) &(20,5) &20 &0(fixed) &5\% &350 &1e-2 &.5per20(step)&5 &16\\
\bottomrule 
\end{tabular}}}\vspace{-0.2cm}
\centering
\caption{Hyper-parameter Setting of \mname} \label{tb:hyper-parameter}
\end{table*}

\begin{table*}[!ht]
\centerline{
\resizebox{\textwidth}{!}{
\begin{tabular}{lccccccccc}
\toprule 
& \textit{\# Classes} &RNN  &LSTM &1D-CNN &Transformer &SNN &RSNN &LSNN &\mname\\
\midrule 
Coffee &2   &.998±.001(4) &1.00±.002\textbf{(1)} &1.00±.002\textbf{(1)} &.998±.001(4)  &.986±.002(6)  &.986±.001(6)  &.986±.004(6) &1.00±.002\textbf{(1)} \\
GunPoint &2 &.987±.001(5) &1.00±.003\textbf{(1)} &1.00±.000\textbf{(1)} &1.00±.003\textbf{(1)}  &.986±.004(7)  &.987±.002(5)  &.986±.003(7) &1.00±.003\textbf{(1)} \\
MoteStrain &2 &.892±.008(3) &.895±.010\textbf{(1)} &.890±.010(5) &.891±.014(4) &.865±.015(7) &.869±.014(6) &.865±.016(7)   &.894±.011(2)\\
Computers &2 &.844±.010(4)  &.849±.011(2) &.851±.009\textbf{(1)} &.843±.012(5) &.803±.010(8) &.809±.012(6) &.806±.012(7) &.849±.006(2) \\
Wafer &2 &.999±.000(3) &.999±.000(3) &.999±.000(3) &1.00±.000\textbf{(1)} &.977±.001(6)  &.977±.001(6)  &.977±.001(6)   &1.00±.000\textbf{(1)}\\
Lightning&2  &.831±.006(3)  &.833±.006(2) &.830±.007(5) &.834±.008\textbf{(1)} &.806±.013(8) &.815±.010(7) &.816±.014(6) &.831±.006(3)\\
Yoga &2 &.875±.012(6) &.886±.011(3) &.886±.013(3) &.886±.011(3) &.875±.012(6) &.875±.016(6) &.888±.014(2)  &.901±.010\textbf{(1)} \\
CBF  &3 &1.00±.000\textbf{(1)} &1.00±.000\textbf{(1)} &1.00±.000\textbf{(1)} &.998±.001(5)  &.996±.002(6)  &.996±.002(6)  &.996±.003(6) &1.00±.009\textbf{(1)} \\
BME &3 &.997±.001(4) &.977±.001(4) &1.00±.000\textbf{(1)} &.999±.001(3)  &.976±.002(7)  &.977±.001(4)  &.976±.002(7) &1.00±.009\textbf{(1)} \\
Trace &4 &1.00±.000\textbf{(1)} &1.00±.000\textbf{(1)} &.999±.001(5) &1.00±.000\textbf{(1)} &.999±.001(5) &.999±.001(5) &.999±.001(5) &1.00±.000\textbf{(1)} \\
Oliveoil &4 &.899±.011(6) &.923±.009(3) &.923±.008(3) &.931±.012\textbf{(1)} &.868±.013(8) &.879±.013(7) &.906±.010(5)  &.924±.008(2)\\
Beef &5 &.835±.012(5) &.844±.013(4) &.848±.011(2)  &.845±.011(3)  &.805±.015(8) &.810±.012(6) &.808±.014(7)   &.849±.011\textbf{(1)} \\
Worms &5 &.709±.011(6) &.709±.012(6) &.723±.015(2) &.723±.010(2)  &.709±.014(6) &.710±.018(5) &.711±.012(4)  &.724±.015\textbf{(1)} \\
Symbols &6 &.954±.005(5) &.962±.006(2)  &.963±.004\textbf{(1)} &.957±.002(4) &.905±.009(8) &.915±.009(6) &.912±.007(7)  &.961±.003(3) \\
Synthetic &6 &1.00±.000\textbf{(1)} &1.00±.000\textbf{(1)} &1.00±.000\textbf{(1)} &.998±.001(5)  &.996±.001(6)  &.996±.002(6)  &.996±.001(6) &1.00±.009\textbf{(1)}  \\

\textit{Rank} & &3.6 &2.3 &2.3 &2.7  &6.8 &5.7 &6.2 &\textbf{1.5}\\
\midrule 
SEPSIS  &2 &.853±.015(7) &.855±.013(4) &.858±.015(2)  &.854±.011(6) &.829±.012(8) &.856±.009(3) &.855±.013(4)  &\textbf{.870±.010(1)}\\
\midrule 
COVID-19 &2  &.963±.013(6)  &.968±.013(2)  &.954±.013(5)  &.941±.012(8) &.942±.010(7)   &.955±.009(3)  &.955±.011(3)    &\textbf{.969±.008(1)}\\
\bottomrule 
\end{tabular}}}\vspace{-0.2cm}
\centering
\caption{Classification Accuracy $\uparrow$ and Performance Ranking $\downarrow$ of Methods} \label{tb:classification accuracy}
\end{table*}

\section{Experiments} \label{sec:experiments}

\subsection{Experimental Setup}
\textsc{Datasets.} The benchmark UCR archive \cite{UCRArchive2018} has univariate and regularly-sampled TS data, consisting of 128 TS datasets. We select 15 datasets covering multiple data types (spectro, sensor, simulated, motion) and classification tasks (binary-, three-, four-, five-, six-classification). Both two real-world healthcare datasets have multivariate and irregularly-sampled TS data: SEPSIS dataset \cite{DBLP:conf/cinc/ReynaJSJSWSNC19} has 30,336 records with 2,359 diagnosed sepsis. Each TS sample has 40 related patient features. Early diagnosis is critical to improving sepsis outcome \cite{seymour2017time}; COVID-19 dataset \cite{COVID-19} has 6,877 blood samples of 485 COVID-19 patients from Tongji Hospital, Wuhan, China. Each sample has 74 laboratory test features. Mortality prediction helps for timely treatment and allocation of medical resources \cite{DBLP:journals/BMC/sun}. 

\textsc{Baselines.} The SOTA ANNs include RNN, LSTM \cite{pmlr-v119-zhao20c}, 1D-CNN \cite{pmlr-v139-jiang21d}, and Transformer \cite{10.1145/3447548.3467401}. The SOTA SNNs include SNN \cite{DBLP:conf/ijcnn/FangSQ20}, RSNN \cite{SRNN}, and LSNN \cite{DBLP:conf/nips/BellecSSL018}. SNNs (SNN, RSNN, LSNN) are the spiking neuron versions of the classical ANNs (MLP, SNN, LSTM). The basic CTS model in \mname is RSNN.

\textsc{Hyper-parameters.} SNNs need the initialization of both the weight and the hyper-parameters of the spiking neurons (time constants, thresholds, starting potential). We initialize the starting value of the membrane potential $V_{i}^{l}[0]$ is initialized with a random value distributed uniformly in the range $[0,a+1.8\eta]$, $\eta$ is the learning rate; We randomly initialize the time constants $\tau_{m},\tau_{s},\tau$ following a tight normal distribution $\mu,\sigma$ with constant, uniform, and normal initializers; We test the pacing parameters in range $[2\%,20\%]$ for start percent $sp$ and range $[50, 2500]$ for step size $ss$; We test the internal number of \mnametwo in range $[5,10]$ and the neuron number of a cluster in $\{8,16,32,64\}$. Table \ref{tb:hyper-parameter} lists the final hyper-parameter setting. This setting also confirms that the assumptions $\frac{\mu}{2}<V_{th}<\mu$ we made in proving Theorem \ref{theo:app_snn1} often exist in practice. Figure \ref{fig:result2}(c) shows that the normal initializer achieves the best performance. \mnameone, \mnametwo, and the typical initializer make SNNs more active and less sparse during the initial training stage, allowing converging faster.

\begin{figure*}[t]
\setlength{\abovecaptionskip}{0cm}
\centerline{
\includegraphics[width=\linewidth]{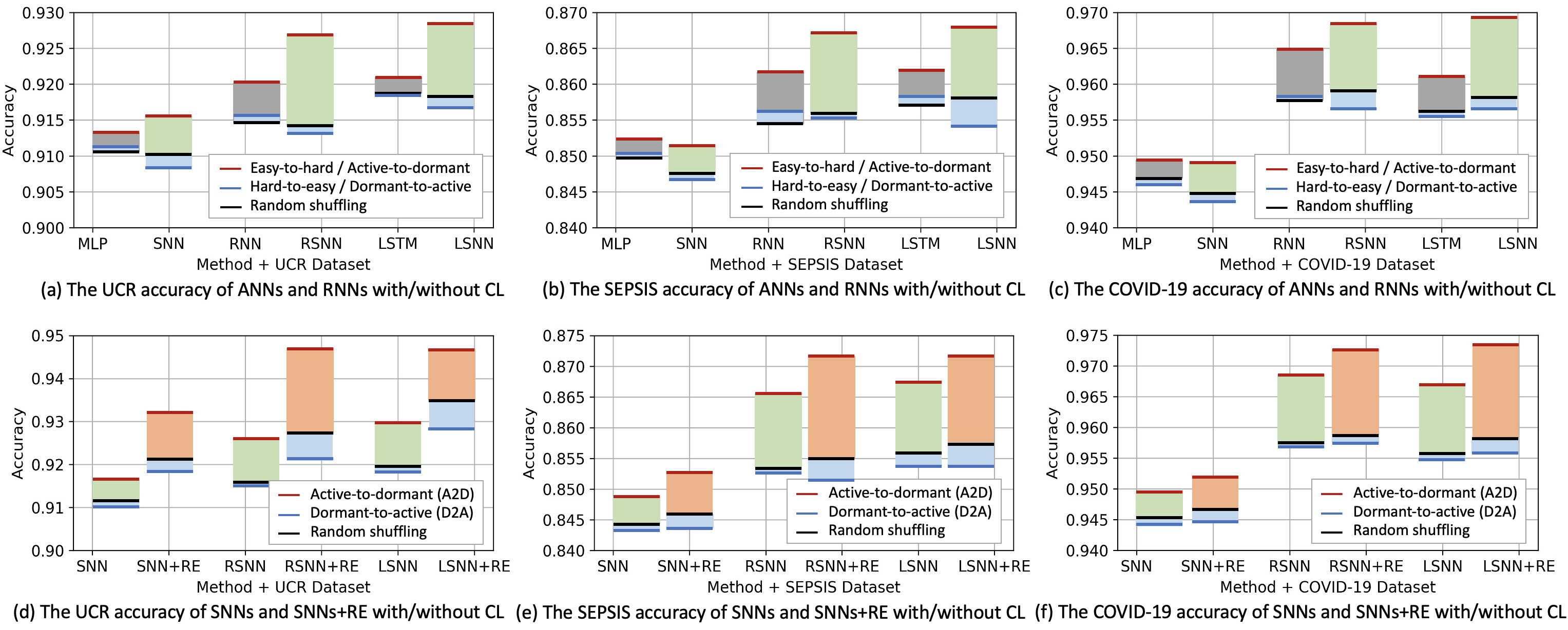}}\vspace{-0.3cm}
\caption{Changes in Classification Accuracy of Different ANNs and SNNs after Applying Curriculum Learning}
\label{fig:result1}
\end{figure*}

\subsection{Classification Accuracy}

\mname significantly improves the accuracy of SNNs for CTS task and makes SNNs (SNN, RSNN, LSNN) achieve performance comparable to deep ANNs (RNN, LSTM, 1D-CNN, Transformer). The 5-fold cross-validation method yields results that are presented as mean $\pm$ standard deviation. The classification accuracy is evaluated by Area Under Curve of Receiver Operating Characteristic (AUC-ROC, the higher the better) and its Confidence Interval (AUC-CI).

\mname has the highest classification accuracy on most datasets among all tested methods as shown in Table \ref{tb:classification accuracy}. It has a significant improvement in the average accuracy according to the Bonferroni-Dunn test with $\alpha=0.05$, $2.949\sqrt{\frac{7\times(7+1)}{6\times5\times3}}=2.62$ (critical difference) $<4.74$ (average rank of baselines). \mname can produce stable results as the 5-fold cross-validation results are all within the AUC-CI. For example, the accuracy of COVID-19 mortality classification ($0.971\pm0.008$) is in the AUC-CI ($0.971\pm 0.012$). 

SNNs are more suitable for modeling irregularly-sampled TS than ANNs. SNNs perform better than ANNs on SEPSIS and COVID-19 datasets. ANNs ignore the effect of uneven time intervals on value dependence, whereas SNNs have a decay mechanism over time for this issue: In inference, the model can be simulated in an event-driven manner, i.e. computation is only required when a spike event occurs. Thus, TS data is not required to have a uniform time interval as SNNs calculate based on time records. Suppose at $t$, $M[t]$ is known. In Equation \ref{eq:timedecay}, after $\Delta t$ unit time later, i.e. at time $t'=t+\Delta t$, $M[t]$ decays over time without input spike; $M[t]$ has an instantaneous unit charge with an input spike. And the similar update rule can also apply for states $H[t]$ and $R[t]$ with $H[t']=H[t]e^{\frac{-\Delta t}{\tau_{s}}}, R[t']=H[t]e^{\frac{-\Delta t}{\tau}}$.
\begin{equation}  \label{eq:timedecay}
\small
\begin{aligned}
    &\mathrm{No\ input\  spike:}\ \ M[t']=\sum_{t_{i}<t}^{t_{i}}e^{-\frac{t+\Delta t-t_{i}}{\tau_{m}}}=M[t]e^{\frac{-\Delta t}{\tau_{m}}}    \\
    & \mathrm{An\ input\ spike:}\ \ M[t']=M[t]e^{\frac{-\Delta t}{\tau_{m}}}+1
\end{aligned}   
\end{equation}

\begin{figure*}[t]
\setlength{\abovecaptionskip}{0cm}
\centerline{
\includegraphics[width=\linewidth]{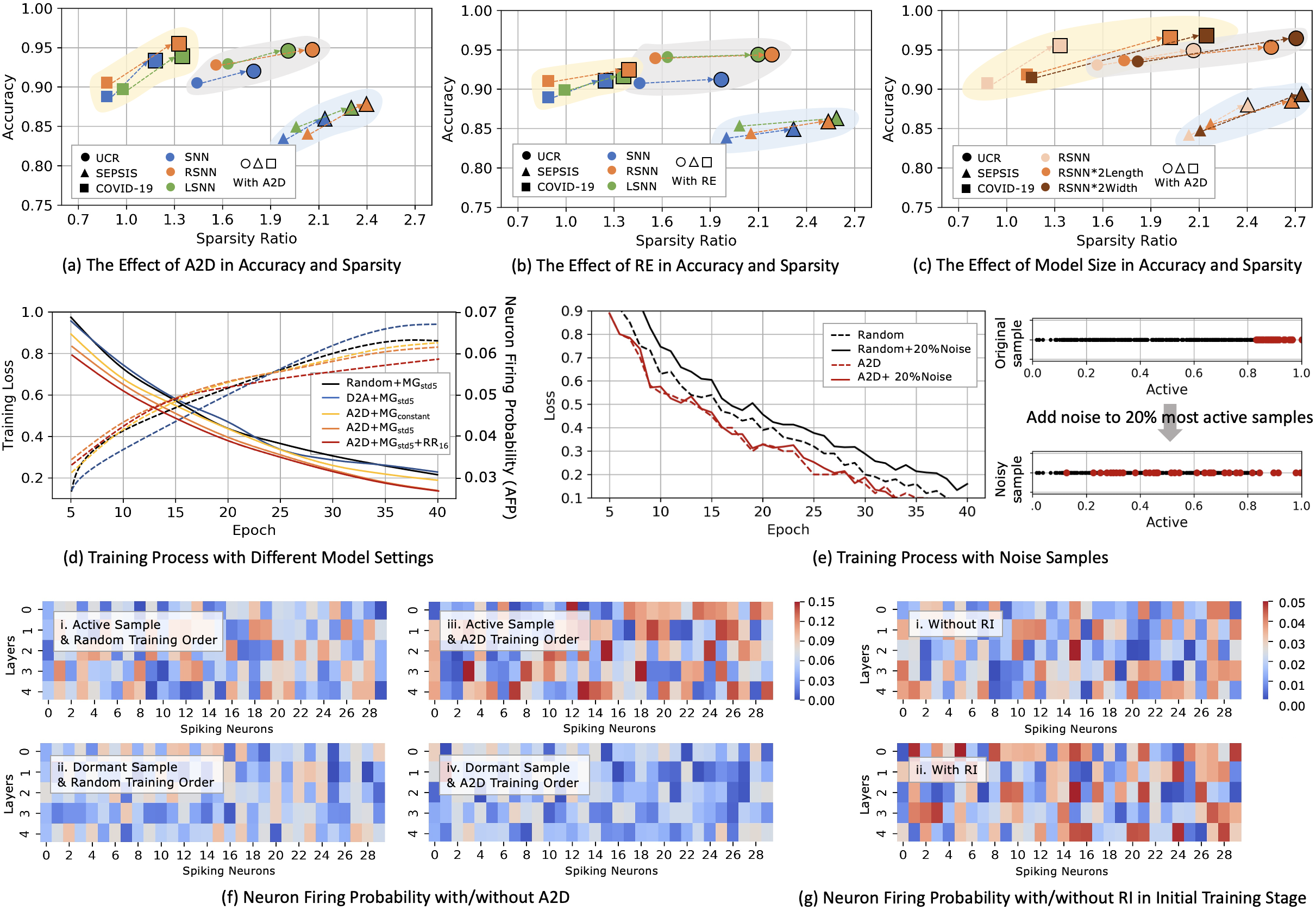}}\vspace{-0.3cm}
\caption{Improvements in Model Sparsity, Firing Statutes, Convergence Process, and Anti-noise Performance after Using \mname}
\label{fig:result2}
\end{figure*}

\subsection{Curriculum Learning Performance}

\subsubsection*{The power of active-to-dormant training order mechanism (\mnameone)}

The orderly training has a greater impact on SNNs than ANNs with about twice the accuracy change. And \mname can improve the classification accuracy, network sparsity, firing status, and anti-noise ability of SNNs. In this experiment, we apply our \mnameone to SNNs and the easy-to-hard training order (E2H) \cite{DBLP:conf/icml/HacohenW19} to ANNs. The dormant-to-active training order (D2A) and the hard-to-easy training order (H2E) are their anti-curriculum. 

The CTS accuracy of SNNs using \mnameone is improved more than that of ANNs using E2H (The green rectangle is bigger than the gray rectangle) as shown in Figure \ref{fig:result1}(a)-(c). Meanwhile, the anti-curriculum will negatively affect SNNs (The blue line is below the black line), but sometimes positively affect ANNs (The blue line is above the black line). This demonstrates that SNNs are more impacted by CL than ANNs are (The distance between the red line and the blue line of SNNs is farther than that of ANNs).

SNNs communicate sparingly. \mnameone can adjust the firing status and network sparsity of the trained SNNs when classify TS data as shown in Figure \ref{fig:result2}(a). The sparsity ratio represents the proportion of neurons that have not been fired in total spiking neurons. After using \mnameone, both the CTS accuracy and the sparsity ratio are increased, where RSNN's sparsity changes the most among the tested SNNs. And the sparsity of SNNs with wider structure is more affected than that of SNNs with wider structure as shown in Figure \ref{fig:result2}(c). This finding opens up more opportunities for using optimization techniques like pruning and quantization, emphasizing the benefits of being lightweight even more.

From the standpoint of the model change, \mnameone improves the CTS accuracy and accelerates the model convergence by changing the neuron activation state of SNNs in the model training process. During training, the neural activity is expressed as the average firing probability per timestep per neuron (AFP). Most SNNs exhibit less than $0.10$ AFP. A2D makes the firing probability higher in the early training stage and lower in the late training stage as shown in Figure \ref{fig:result2}(d), meaning the model will be more activated at the start of training. This accelerates the network convergence: A2D improves the accuracy over random training order with the same number of training epochs. While this is happening, wider networks make the phenomenon more obvious.

From the standpoint of the sample activity, \mnameone improves CTS accuracy by increasing neuron firing probability when inputting the active samples and decreases that when inputting the dormant samples compared with the random training order as shown in Figure \ref{fig:result2}(f). On the one hand, it can increase the network's initial convergence speed and, on the other hand, can lessen the network's exposure to noise data by treating noise data as dormant samples.

\mnameone gives SRNN a specific anti-noise ability as shown in Figure \ref{fig:result2}(e). We add Gaussian noise with a signal-to-noise ratio of $20\mathrm{db}$ to the most $20\%$ active TS samples in \mnameone. After recalculating the activity, these noisy samples become less active and the priority of participating in training is reduced. During training process, the impact of noise on SRNN's loss is significantly reduced after using \mnameone.

\subsubsection*{The power of value-based regional encoding mechanism (\mnametwo)}

The CTS accuracy of SNNs with \mnametwo is more accurate than that of SNNs without \mnametwo (The black lines of SNN+\mnametwo, RSNN+\mnametwo, and LSNN+\mnametwo are above the black lines of SNN, RSNN, and LSNN) as shown in Figure \ref{fig:result1}(d)-(f). Meanwhile, when using \mnametwo, \mnameone will affect SNNs more positively (The orange rectangle of SNNs+\mnametwo is bigger than green rectangle of SNNs), and anti-curriculum D2A will affect SNNs less negatively (The blue rectangle of SNNs+\mnametwo is smaller than the blue rectangle of SNNs).

\mnametwo can increase the sparsity ratio of SNNs more than \mnameone, although it does not enhance accuracy as much, shown in Figure \ref{fig:result2}(b). \mnametwo can also accelerate model convergence by increasing the neuron activity in the initial model training stage as shown in Figure \ref{fig:result2}(g).

\mnametwo works better with univariate TS than it does with multivariate TS as shown in Figure \ref{fig:result1}(d)-(f). For example, compared to the multivariate SEPSIS and COVID-19 datasets, the accuracy improvement in the univariate UCR dataset is greater (The lifting between the black line of SNNs+\mnametwo and that of SNNs in (d) is greater than that in (e) and (f)). This might be because the region division criterion for multivariate TS is based on the mean value of all univariate, which may result in a lesser distinction between regions.

\section{Conclusion}

This paper investigates the power of curriculum learning (CL) on spiking neural networks (SNNs) for the classification of time series (CTS) for the first time. We design a curriculum for SNNs (\mname) by proposing an active-to-dormant training order mechanism (\mnameone) and a value-based regional encoding mechanism (\mnametwo). Through the theoretical analysis and experimental confirmation, we reach the following results: The constructed SNNs have a greater potential for modeling TS data than ANNs. Because each neuron is a recursive system, it can represent the real-world irregularly-sampled TS without the need for any additional mechanisms; Compared to ANNs, CL affects SNNs and wider SNN structures more positively, whereas the anti-curriculum may affect SNNs negatively; The designed SNNs' curriculum \mname can simulate the order of human knowledge acquisition and how the brain processes sequential input. It is appropriate for the properties of spiking neurons. By adjusting the spiking neuron firing statutes and activities, \mname can improve the CTS accuracy, model convergence speed, network sparsity, and anti-noise ability of SNNs.


\newpage

\bibliography{references}
\bibliographystyle{icml2022}

\newpage
\appendix
\onecolumn

\setcounter{theorem}{0} 
\setcounter{proposition}{0} 

\section{Theorems, Propositions, and Proofs } \label{sec:appendix}

\begin{theorem} \label{theo:app_snn1}
In SNN, different input orders of spike trains make the neuron output different spike trains.
\end{theorem}

\begin{proof}
We assume two input spike trains with the same spike number and equal time intervals, most simply, an $l$-length all-one spike sequence $S_{0}=(1)^{l}$, and an $l$-length half-one spike sequence $S_{1}=(0)^{\frac{l}{2}}\land (1)^{\frac{l}{2}}$. When the SNN training starts, the parameters $w$ are initialized randomly under a Gaussian distribution $\mathcal{N}(\mu,\sigma^2)$. We focus on the membrane potential $v$ of one neuron of the first layer and assume $\frac{\mu}{2}<V_{th}<\mu$.

If the learning order is $S_{0} \rightarrow S_{1}$: After inputting $S_{0}$, the membrane potential is $v=\sum_{i=1}^{l} w\mathrm{PSP}\approx\mu>V_{th}$, then the neuron fires and $w\rightarrow w'$; After inputting $S_{1}$, $v\approx \frac{\mu'}{2}$.

If the learning order is $S_{1} \rightarrow S_{0}$: After inputting $S_{1}$, $v\approx \frac{\mu}{2}<V_{th}$, then the neuron does not fire and $w$ remains; After inputting $S_{0}$, $v\approx \mu >V_{th}$, then the neuron fires.

When $\frac{\mu'}{2}>V_{th}$, there will be two fires in the first case but one in the second case $O_{0}=(1,1)\neq O_{1}=(0,1)$.
\end{proof}

\begin{theorem} \label{theo:app_cl1}
The difference between objectives $\mathcal{U}_{p}$ and $\mathcal{U}$, which are computed with and without curriculum prior $p$, is the covariance between $U_{\vartheta}$ and $p$.
\begin{equation} 
\small
    \mathcal{U}_{p}(\vartheta)=\mathcal{U}(\vartheta)+\hat{\mathrm{Cov}}[U_{\vartheta},p]
\nonumber
\end{equation} 
\end{theorem}

\begin{proof}
By deforming the objective $\mathcal{U}_{p}(\vartheta)=\hat{\mathbb{E}}_{p}[U_{\vartheta}]=\sum_{i=1}^{N}U_{\vartheta}(X_{i})p(X_{i})$, we can find that $\mathcal{U}_{p}(\vartheta)$ is determined by the correlation between $U_{\vartheta}(X)$ and $p(X)$.
\begin{equation} 
\small
\begin{aligned}
     \mathcal{U}_{p}(\vartheta)=&\sum_{i=1}^{N}U_{\vartheta}(X_{i})p(X_{i})\\
     =&\sum_{i=1}^{N}U_{\vartheta}(X_{i})p(X_{i})-2N\hat{\mathbb{E}}(U_{\vartheta})\hat{\mathbb{E}}(p)+2N\hat{\mathbb{E}}(U_{\vartheta})\hat{\mathbb{E}}(p)\\
     =&(\sum_{i=1}^{N}U_{\vartheta}(X_{i})p(X_{i})-N\hat{\mathbb{E}}(p)\sum_{i=1}^{N}U_{\vartheta}(X_{i})\\&-N\hat{\mathbb{E}}(U_{\vartheta})\sum_{i=1}^{N}p(X_{i})+N\hat{\mathbb{E}}(U_{\vartheta})\hat{\mathbb{E}}(p))+N\hat{\mathbb{E}}(U_{\vartheta})\hat{\mathbb{E}}(p)\\
     =&\sum_{i=1}^{N}(U_{\vartheta}(X_{i})-\hat{\mathbb{E}}[U_{\vartheta}])(p_{i}-\hat{\mathbb{E}}[p])+N\hat{\mathbb{E}}(U_{\vartheta})\hat{\mathbb{E}}(p)\\
    =&\hat{\mathrm{Cov}}[U_{\vartheta},p]+\mathcal{U}(\vartheta)
\end{aligned}
\nonumber
\end{equation}
\end{proof}

\begin{theorem} \label{theo:app_cl2}
Curriculum learning changes the optimization function from $\mathcal{U}(\vartheta)$ to $\mathcal{U}_{p}(\vartheta)$. The optimal $\Tilde{\vartheta}$ maximizes the covariance between $p$ and $U_{\Tilde{\vartheta}}$:
\begin{equation}
\small
    \Tilde{\vartheta}=\arg\max\limits_{\vartheta}\mathcal{U}(\vartheta)=\arg\max\limits_{\vartheta}\hat{\mathrm{Cov}}[U_{\vartheta},p] 
\nonumber
\end{equation}
and satisfying two claims:
\begin{equation}
\small
    \begin{aligned}
        &\Tilde{\vartheta}=\arg\max\limits_{\vartheta}\mathcal{U}(\vartheta)=\arg\max\limits_{\vartheta}\mathcal{U}_{p}(\vartheta)\\
        &\forall\vartheta \quad \mathcal{U}_{p}(\Tilde{\vartheta})-\mathcal{U}_{p}(\vartheta)\geq \mathcal{U}(\Tilde{\vartheta})-\mathcal{U}(\vartheta)
    \end{aligned} 
\nonumber
\end{equation}
\end{theorem}

\begin{proof} For claim 2, from Theorem \ref{theo:app_cl1},
    \begin{equation}
    \small
    \begin{aligned}
        \mathcal{U}_{p}(\Tilde{\vartheta})-\mathcal{U}_{p}(\vartheta)&=\mathcal{U}_{p}(\Tilde{\vartheta})-\mathcal{U}(\vartheta)-\hat{\mathrm{Cov}}[U_{\vartheta},p]\\
        &\geq \mathcal{U}_{p}(\Tilde{\vartheta})-\mathcal{U}(\vartheta)-\hat{\mathrm{Cov}}[U_{\Tilde{\vartheta}},p]\\
        &=\mathcal{U}(\Tilde{\vartheta})-\mathcal{U}(\vartheta)
    \end{aligned}  
\nonumber
    \end{equation}
\end{proof}

\begin{proposition}\label{prop:app_1}
    When using the active-to-dormant training order, Theorem \ref{theo:app_cl2} holds if the variance of the maximum function is roughly constant in the relevant range of plausible parameter values.
\begin{equation}\small
\begin{aligned}
    \mathcal{U}_{p}(\Tilde{\vartheta})-\mathcal{U}_{p}(\vartheta)\geq \mathcal{U}(\Tilde{\vartheta})-\mathcal{U}(\vartheta)\quad \forall\vartheta: \mathrm{Cov}[U_{\vartheta},U_{\Tilde{\vartheta}}]\leq\mathrm{Var}[U_{\Tilde{\vartheta}}]
\end{aligned}\nonumber
\end{equation}
\end{proposition}

\begin{proof}
Based on Equation \ref{eq:p}, $\mathcal{U}_{p}(\vartheta)=\mathcal{U}(\vartheta)+\frac{1}{P}\mathrm{Cov}[U_{\vartheta},U_{\Tilde{\vartheta}}]$. Then, at the optimal point $\Tilde{\vartheta}$: $\mathcal{U}_{p}(\Tilde{\vartheta})=\mathcal{U}(\Tilde{\vartheta})+\frac{1}{P}\mathrm{Var}[U_{\Tilde{\vartheta}}]$; at any other point: $\mathcal{U}_{p}(\vartheta)\leq\mathcal{U}(\Tilde{\vartheta})+\frac{1}{P}\sqrt{\mathrm{Var}[U_{\vartheta}]\mathrm{Var}[U_{\Tilde{\vartheta}}]}$. 

The sample itself is randomly distributed. Assuming a constant $b=\mathrm{Var}[U_{\vartheta}]$, Theorem \ref{theo:cl2} follows $\mathcal{U}_{p}(\vartheta)\leq\mathcal{U}(\Tilde{\vartheta})+\frac{b}{P}=\mathcal{U}_{p}(\Tilde{\vartheta})$, i.e. $\Tilde{\vartheta}=\arg\max\limits_{\vartheta}\mathcal{U}_{p}(\vartheta)$.
\end{proof}

\begin{theorem} \label{theo:app_ri1}
The difference between the objective score with CL and that without CL will be greater if there is a larger score discrepancy between each sample.
\end{theorem}

\begin{proof} 
Without CL, since the training samples are fixed, $\small{\mathrm{Var}[U_{\vartheta}]}$ is constant; With CL, based on Equation \ref{eq:p}, as $\small{P \approx \mathbb{E}[e^{-L(X_{i})}]}$ is constant, $\small{\mathrm{Var}[U_{\Tilde{\vartheta}}]\propto \mathrm{Var}[p]}$. Thus, 
\begin{equation}
\begin{aligned}
    \mathrm{Var}[p]\uparrow&\rightarrow
    \mathrm{Var}[U_{\Tilde{\vartheta}}]\uparrow\\
    &\rightarrow \mathbb{E}[(U_{\theta}(X_{i})-\mathbb{E}(U_{\theta}))(U_{\Tilde{\vartheta}}(X_{i})-\mathbb{E}(U_{\Tilde{\vartheta}}))\downarrow]\\
    &\rightarrow 
    \mathrm{Cov}[U_{\vartheta},U_{\Tilde{\vartheta}}]\downarrow
\end{aligned}
    \nonumber
\end{equation}

\end{proof}

\begin{proposition}\label{prop:app_2}
    The regional encoding mechanism contributes to SNNs' CL by increasing the variation of sample activity and assisting the model's quick response to input changes.
\end{proposition}

\begin{proof}
    At time $t$, we call spiking neurons with $x_{t}\in \mathcal{I}_{m}$ as excitatory neurons $\mathrm{S}_{E}$ and neurons with $x_{t}\notin \mathcal{I}_{m}$ as inhibitory neurons $\mathrm{S}_{I}$. For $\mathrm{S}_{I}$, the membrane potential is only discharged but not charged, i.e. Equation \ref{eq:vt} only has the rightmost item. We can regard it as an extreme case of feed-forward and feed-back inhibition \cite{2022BackEISNN}. The balance of excitation and inhibition can achieve fast-response to input changes \cite{DBLP:conf/smc/TianHW19}. 
    
    Based on \cite{2022BackEISNN}, the inhibition changes piking rate, the value of Equation \ref{eq:U} increases under the true class label, so that the difference among computed values of Equation \ref{eq:loss_L_theta} of samples with different labels becomes larger. Thus, $\mathrm{Var}[U_{\Tilde{\vartheta}}]$ increases.
    
    According to Equation \ref{eq:SNN}, the neuronal input is given by external input and recurrent input $I_{i}[t]=I^{ext}_{i}[t]+I^{rec}_{i}[t]$. With a probability $P_{C}$, each neuron connecting the input layer receive inputs from $K_{E}=P_{C} \mathrm{S}_{E}$ excitatory neurons and $K_{I}=P_{C} \mathrm{S}_{I}$ inhibitory. The recurrent excitatory (inhibitory) input can be approximated by a Poisson process \cite{DBLP:conf/icc/TorabK01}. a Poisson presynaptic spike train is $I_{i}^{rec}[t]=\mu_{i}^{rec}+\sigma_{i}^{rec}\xi_{i}^{rec}$. Hereafter, we omit the neuron index $i$, and use the mpty subscript $e/i=E$ or $I$ to denote whether the population is excitatory or inhibitory. $\mu_{e/i}=\mu_{rec,e/i}+f_{e/i}\mu_{ext}$, $\sigma^{2}_{e/i}=\sigma^{2}_{ext}+\sigma^{2}_{rec,e/i}$, and $\beta_{e/i}=\frac{\sigma^{2}_{e/i}}{\mu_{e/i}}$ are mean, variance, and variance-to-mean radio of the input received by a neuron. In reality, $\sigma^{2}_{ext}$ is usually very small and $\beta_{e/i}$ can be approximated as a constant irrespective to external inputs. Referring to the derivation process in \cite{DBLP:conf/smc/TianHW19}, we directly draw a conclusion: The neural population firing rate is Equation \ref{eq:r}. It linearly encodes the external input mean. Which ensures that the network’s response to input changes is very fast.
    \begin{equation}\label{eq:r}
        r_{e/i}=\frac{\mu_{e/i}}{V_{th}V_{0}}\propto \mu_{ext}
    \end{equation}
\end{proof}


\end{document}